\documentclass{article}

\PassOptionsToPackage{numbers, compress}{natbib}

\usepackage[acronym,smallcaps,nowarn]{glossaries}
\glsdisablehyper{}

\newacronym{RL}{rl}{reinforcement learning}
\newacronym{SIL}{sil}{self-imitation learning}
\newacronym{PPO}{ppo}{proximal policy optimization}
\newacronym{TD3}{td3}{twin-delayed deep deterministic policy gradient}
\newacronym{DDPG}{ddpg}{deep deterministic policy gradient}
\newacronym{TD}{td}{temporal difference}

\newacronym{EM}{em}{expectation maximization}
\newacronym{IS}{is}{importance sampling}

\newacronym{HER}{her}{hindsight experience replay}
\newacronym{HPG}{hpg}{hindsight policy gradient}
\newacronym{hEM}{$\mathrm{h}$em}{hindsight expectation maximization}

\newacronym{MDP}{mdp}{markov decision process}
\newacronym{Variational RL}{variational rl}{variational reinforcement learning}

\newacronym{ELBO}{elbo}{evidence lower bound}

 \usepackage[final]{nips_2018}
\usepackage[utf8]{inputenc} 
\usepackage[T1]{fontenc}    
\usepackage{hyperref}       
\usepackage{url}            
\usepackage{booktabs}       
\usepackage{amsfonts}       
\usepackage{nicefrac}       
\usepackage{microtype}      
\usepackage{amsmath}
\usepackage{algorithm,algorithmic}
\usepackage{graphicx}
\usepackage{amsmath,amssymb,amsthm}
\usepackage{algorithm,algorithmic}
\usepackage{subfigure}

\usepackage{xcolor}
\usepackage{framed}
\usepackage{graphicx}
\usepackage{mathtools}
\graphicspath{ {images/} }
\usepackage{tikz}
\usetikzlibrary{arrows,shapes,snakes,automata,backgrounds,petri}
\usepackage{pdfpages}
\usepackage{hyperref}
\hypersetup{
	colorlinks=true,
	linkcolor=blue,
	filecolor=magenta,      
	urlcolor=cyan,
}

\usepackage[T1]{fontenc}
\usepackage[utf8]{inputenc}
\usepackage{tabularx,ragged2e,booktabs,caption}
\newcolumntype{C}[1]{>{\Centering}m{#1}}

\usepackage{caption}

\usepackage{enumitem}

\usepackage{thmtools}
\usepackage{thm-restate}
\usepackage{hyperref}
\usepackage{cleveref}

\newtheorem{predefinition}{Definition}

\newtheorem{theorem}{Theorem}

\newtheorem{preproposition}{Proposition}

\usepackage{wrapfig}
\usepackage{mathtools}

\newtheorem{lemma}{Lemma}

\title{Self-Imitation Learning via Generalized Lower Bound Q-learning}

\author{%
  Yunhao Tang \\
  Columbia University\\
  \texttt{yt2541@columbia.edu}
  }

\begin{document}

\maketitle

\begin{abstract}
Self-imitation learning motivated by lower-bound Q-learning is a novel and effective approach for off-policy learning. In this work, we propose a n-step lower bound which generalizes the original return-based lower-bound Q-learning, and introduce a new family of self-imitation learning algorithms. To provide a formal motivation for the potential performance gains provided by self-imitation learning, we show that n-step lower bound Q-learning achieves a trade-off between fixed point bias and contraction rate, drawing close connections to the popular uncorrected n-step Q-learning. We finally show that n-step lower bound Q-learning is a more robust alternative to return-based self-imitation learning and uncorrected n-step, over a wide range of continuous control benchmark tasks.
\end{abstract}

\section{Introduction}
Learning with off-policy data is of central importance to scalable \gls{RL}. The traditional framework of off-policy learning is based on \gls{IS}: for example, in policy evaluation, given trajectories $(x_t,a_t,r_t)_{t=0}^\infty$ generated under behavior policy $\mu$, the objective is to evaluate Q-function $Q^\pi(x_0,a_0)$ of a target policy $\pi$. Naive \gls{IS} estimator involves products of the form $\pi(a_t\mid x_t) / \mu(a_t\mid x_t)$ and is infeasible in practice due to high variance. To control the variance, a line of prior work has focused on operator-based estimation to avoid full \gls{IS} products, which reduces the estimation procedure into repeated iterations of off-policy evaluation operators \citep{precup2001off,harutyunyan2016q,munos2016safe}. Each iteration of the operator requires only local \gls{IS} ratios, which greatly stabilizes the update. 

More formally, such operators $\mathcal{T}$ are designed such that their fixed points are the target Q-function $\mathcal{T} Q^\pi = Q^\pi$. As such, these operators are \emph{unbiased} and conducive to theoretical analysis. However, a large number of prior work has observed that certain \emph{biased} operators tend to have significant empirical advantages \citep{hessel2018rainbow,barth2018distributed,kapturowski2018recurrent}. One notable example is the uncorrected $n$-step operator, which directly bootstraps from $n$-step target trajectories without \gls{IS} corrections \citep{hessel2018rainbow}. The removal of all \gls{IS} ratios biases the estimate, but allows the learning signal to be propagated over a longer horizon (in Section 2, we will characterize such effects as contraction rates). Indeed, when behavior trajectories are unlikely under the current policy, small \gls{IS} ratios $\pi(a_t\mid x_t) / \mu(a_t\mid x_t)$ quickly cut off the learning signal. In general, there is a trade-off between the fixed point bias and contraction rates. Empirical findings suggest that it might be desirable to introduce bias in exchange for faster contractions in practice \citep{rowland2019adaptive}.


Recently, \gls{SIL} has been developed as a family of novel off-policy algorithms which facilitate efficient learning from highly off-policy data \citep{oh2018self,gangwani2018learning,guo2019efficient}. In its original form, \gls{SIL} is motivated as lower bound Q-learning \citep{he2016learning}. In particular, let $Q_L(x,a) \leq Q^{\pi^\ast}(x,a)$ denote a lower bound of the optimal Q-function $Q^{\pi^\ast}$. Optimizing auxiliary losses which encourage $Q_\theta(x,a) \geq Q_L(x,a)$ could significantly speed up learning with the trained Q-function $Q_\theta(x,a)$. Such auxiliary losses could be extended to actor-critic algorithms with stochastic policies \citep{oh2018self}: \gls{SIL} suggests optimizing a policy $\pi_\theta(a\mid x)$ by maximizing an objective similar to $[Q^\mu(a\mid x) - V^{\pi_\theta}(x)]_+ \log \pi_\theta(a\mid x)$, where $V^{\pi_\theta}(x)$ is the value-function for policy $\pi_\theta$, with $[x]_+\coloneqq\max(0,x)$. The update is intuitively reasonable: if a certain actions $a$ is high-performing under behavior policy $\mu$, such that $Q^\mu(x,a) > V^{\pi_\theta}(x)$, the policy $\pi_\theta(a\mid x)$ should  imitate such actions.

On a high-level, \gls{SIL} is similar to the uncorrected $n$-step update in several aspects. With no explicit \gls{IS} ratios, both methods entail that off-policy learning signals propagate over long horizons without being \emph{cut-off}. As a result, both methods are biased due to the absence of proper corrections, and could be seen as trading-off fixed point bias for fast contractions.

\paragraph{Main idea.} In this paper, we make several theoretical and empirical contributions.
\begin{itemize}[topsep=0pt,parsep=0pt,partopsep=0pt,leftmargin=*]
\item \textbf{Generalized \gls{SIL}}. In Section \ref{sec:sil}, we propose generalized \gls{SIL} which strictly extends the original \gls{SIL} formulation \citep{oh2018self}. Generalized \gls{SIL} provides additional flexibility and advantages over the original \gls{SIL}: it learns from partial trajectories and bootstraps with learned Q-function; it applies to both stochastic and deterministic actor-critic algorithms.
\item \textbf{Trade-off.} In Section \ref{sec:tradeoff}, we formalize the trade-offs of \gls{SIL}. We show that generalized \gls{SIL} trades-off contraction rates with fixed point bias in a similar way to uncorrected $n$-step \citep{rowland2019adaptive}. Unlike uncorrected $n$-step, for which fixed point bias could be either positive or negative, the operator for \gls{SIL} induces \emph{positive} bias, which fits the motivation of \gls{SIL} to move towards optimal Q-functions.
\item \textbf{Empirical.} In Section \ref{section:exp}, we show generalized \gls{SIL} outperforms alternative baseline algorithms.
\end{itemize}


\section{Background}

Consider the standard formulation of \gls{MDP}. At a discrete time $t\geq 0$, an agent is in state $x_t\in \mathcal{X}$, takes action $a_t\in\mathcal{A}$, receives a reward $r_t = r(x_t,a_t) \in \mathbb{R}$ and transitions to a next state $x_{t+1}\sim p(\cdot\mid x_t,a_t)\in \mathcal{X}$. A policy $\pi(a\mid x):\mathcal{X}\mapsto \mathcal{P}(\mathcal{A})$ defines a map from state to distributions over actions. The standard objective of \gls{RL} is to maximize the expected cumulative discounted returns $J(\pi) \coloneqq \mathbb{E}_\pi[\sum_{t\geq 0}\gamma^t r_t]$ with a discount factor $\gamma\in (0,1)$.

Let $Q^\pi(x,a)$ denote the Q-function under policy $\pi$ and $Q^\pi\in \mathbb{R}^{|\mathcal{X}|\times |\mathcal{A}|}$ its vector form. Denote the Bellman operator as $\mathcal{T}^\pi$ and optimality operator as $\mathcal{T}^\ast$ \citep{bellman1957markovian}. Let $\pi^\ast$ be the optimal policy, i.e. $\pi^\ast = \arg\max_\pi J(\pi)$. It follows that $Q^\pi,Q^{\pi^\ast}$ are the unique fixed points of $\mathcal{T}^\pi,\mathcal{T}^\ast$ respectively \citep{sutton1999}. Popular \gls{RL} algorithms are primarily motivated by the fixed point properties of the Q-functions (or value functions): in general, given a parameterized Q-function $Q_\theta(x,a)$, the algorithms proceed by minimizing an empirical Bellman error loss $\min_\theta \mathbb{E}_{(x,a)}[(Q_\theta(x,a) - \mathcal{T} Q_\theta(x,a))^2]$ with operator $\mathcal{T}$. Algorithms differ in the distribution over sampled $(x,a)$ and the operator $\mathcal{T}$. For example, Q-learning sets the operator $\mathcal{T} = \mathcal{T}^\ast$ for value iteration and the samples $(x,a)$ come from an experience replay buffer \citep{mnih2013}; Actor-critic algorithms set the operator $\mathcal{T}=\mathcal{T}^{\pi_\theta}$ for policy iteration and iteratively update the policy $\pi_\theta$ for improvement, the data $(x,a)$ could be either on-policy or off-policy \citep{lillicrap2015continuous,schulman2015,wang2016,schulman2017}.

\subsection{Elements of trade-offs in Off-policy Reinforcement Learning}
Here we introduce elements essential to characterizing the trade-offs of generic operators $\mathcal{T}$ in off-policy \gls{RL}. For a complete review, please see \citep{rowland2019adaptive}. Take off-policy evaluation as an example: the data are generated under a behavior policy $\mu$ while the target is to evaluate $Q^\pi$. Consider a generic operator $\mathcal{T}$ and assume that it has fixed point $\tilde{Q}$. Define the contraction rate of the operator as $\Gamma(\mathcal{T}) \coloneqq \sup_{Q_1\neq Q_2} \;\|\; \mathcal{T}(Q_1-Q_2)\;\|\; _\infty / \;\|\; Q_1-Q_2\;\|\; _\infty$. Intuitively, operators with small contraction rate should have \emph{fast} contractions to the fixed point. In practical algorithms, the quantity $\mathcal{T}Q(x,a)$ is approximated via stochastic estimations, denoted as $\tilde{\mathcal{T}}Q(x,a)$. All the above allows us to define the bias and variance of an operator $\mathbb{B}(\mathcal{T}) \coloneqq \;\|\; \tilde{Q} - Q^\pi\;\|\; _2^2, \mathbb{V}(\mathcal{T})\coloneqq \mathbb{E}_\mu [\;\|\; \tilde{\mathcal{T}}Q - \mathcal{T} Q \;\|\; _2^2]$ evaluated at a Q-function $Q$. Note that all these quantities depend on the underlying \gls{MDP} $M$, though when the context is clear we omit the notation dependency.

Ideally, we seek an operator $\mathcal{T}$ with small bias, small variance and small contraction rate. However, it follows that these three aspects could not be optimized simultaneously for a general class of \gls{MDP}s $M \in \mathcal{M}$
\begin{align}
\sup_{M\in \mathcal{M}} \{ \mathbb{B}(\mathcal{T}) + \sqrt{\mathbb{V}(\mathcal{T})} + \frac{2 r_{\text{max}}}{1-\gamma} \Gamma(\mathcal{T}) \} \geq I(\mathcal{M}), \label{eq:trade-off}    
\end{align}
where $r_{\text{max}} \coloneqq \max_{x,a} r(x,a)$ and $I(\mathcal{M})$ is a information-theoretic lower bound \citep{rowland2019adaptive}. This inequality characterizes the fundamental trade-offs of these three quantities in off-policy learning. Importantly, we note that though the variance $\mathbb{V}({\mathcal{T}})$ is part of the trade-off, it is often not a major focus of algorithmic designs \citep{hessel2018rainbow,rowland2019adaptive}. We speculate it is partly because in practice the variance could be reduced via e.g. large training batch sizes, while the bias and contraction rates do not improve with similar techniques. As a result, henceforth we focus on the trade-off between the bias and contraction rate.

\subsection{Trading off bias and contraction rate}
Off-policy operators with unbiased fixed point $\mathbb{B}(\mathcal{T}) = 0$ are usually more conducive to theoretical analysis \citep{munos2016safe,rowland2019adaptive}. For example, Retrace operators $\mathcal{R}_c^{\pi,\mu}$ are a family of off-policy evaluation operators indexed by trace coefficients $c(x,a)$. When $c(x,a) \leq \pi(a\mid x) / \mu(a\mid x)$, these operators are unbiased in that $\mathcal{R}_c^{\pi,\mu} Q^\pi = Q^\pi$, resulting in $\mathbb{B}(\mathcal{R}_c^{\pi,\mu})=0$. One popular choice is $c(x,a) = \min\{\bar{c},\pi(a\mid x) / \mu(a\mid x)\}$ such that the operator also controls variance $\mathbb{V}(\mathcal{R}_c^{\pi,\mu})$ \citep{munos2016safe} with $\bar{c}$.

However, many prior empirical results suggest that bias is not a major bottleneck in practice. For example, uncorrected $n$-step update is a popular technique which greatly improves DQN  \citep{mnih2013} where the \gls{RL} agent applies the operator $\mathcal{T_\text{nstep}^{\pi,\mu}} \coloneqq (\mathcal{T}^\mu)^{n-1}\mathcal{T}^{\pi}$ where $\pi,\mu$ are target and behavior policies respectively \citep{barth2018distributed,kapturowski2018recurrent}. Note that since $\mathcal{T_\text{nstep}^{\pi,\mu}}Q^\pi\neq Q^\pi$, the $n$-step operator is biased $\mathbb{B}(\mathcal{T_\text{nstep}^{\pi,\mu}}) > 0$ \citep{rowland2019adaptive}. However, its contraction rate is small due to uncorrected updates $\Gamma(\mathcal{T_\text{nstep}^{\pi,\mu}}) \leq \gamma^{n}$. On the other hand, though Retrace operators have unbiased fixed point, its contraction rates are typically high due to small \gls{IS}, which \emph{cut off} the signals early and fail to bootstrap with long horizons. The relative importance of contraction rate over bias is confirmed through the empirical observations that $n$-step often performs significantly better than Retrace in challenging domains \citep{kapturowski2018recurrent,rowland2019adaptive}. Such observations also motivate trading off bias and contraction rates in an adaptive way \citep{rowland2019adaptive}.



\subsection{Self-imitation Learning}

\paragraph{Maximum entropy \gls{RL}.} \gls{SIL} is established under the framework of maximum-entropy \gls{RL} \citep{ziebart2010,fox2015taming,asadi2017,nachum2017bridging,haarnoja2018soft}, where the reward is augmented by an entropy term $r_{\text{ent}}(x,a) \coloneqq r(x,a) + c\mathcal{H}^\pi(x)$ and $\mathcal{H}^\pi(x)$ is the entropy of policy $\pi$ at state $x$, weighted by a constant $c > 0$. Accordingly, 
the Q-function is $Q_{\text{ent}}^\pi(x_0,a_0) \coloneqq \mathbb{E}_\pi[r_0 + \sum_{t\geq 1}^\infty \gamma^t (r_t + c \mathcal{H}^\pi(x_t))]$. The maximum-entropy \gls{RL} objective is $J_{\text{ent}}(\pi) \coloneqq \mathbb{E}_\pi [\sum_{t\geq 0}^\infty \gamma^t (r_t + c \mathcal{H}^\pi(x_t)]$. Similar to standard \gls{RL}, we denote the optimal policy $\pi_{\text{ent}}^\ast = \arg\max_\pi J_{\text{ent}}(\pi)$ and its Q-function $Q_{\text{ent}}^{\pi_{\text{ent}}^\ast}(x,a)$. 

\paragraph{Lower bound Q-learning.} Lower bound Q-learning is motivated by the following inequality \citep{oh2018self}, 
\begin{align}
    Q_{\text{ent}}^{\pi_{\text{ent}}^\ast}(x,a) \geq Q_{\text{ent}}^{\mu}(x,a) = \mathbb{E}_\mu[r_0 + \sum_{t\geq 1}^\infty \gamma^t (r_t + c \mathcal{H}^\mu(x_t))],
    \label{eq:sil-lowerbound}
\end{align}
where $\mu$ is an arbitrary behavior policy. Lower bound Q-learning optimizes the following objective with the parameterized Q-function $Q_\theta(x,a)$,
\begin{align}
    \min_\theta \mathbb{E}_\mathcal{D} [([Q^\mu(x,a) - Q_\theta(x,a)]_+)^2], 
    \label{eq:sil-q}
\end{align}
where $[x]_+ \coloneqq \max(x,0)$. The intuition of Eqn.(\ref{eq:sil-q}) is that the Q-function $Q_\theta(x,a)$ obtains learning signals from all trajectories such that $Q_{\text{ent}}^\mu(x,a) > Q_\theta(x,a) \approx Q^{\pi_\theta}(x,a)$, i.e. trajectories which perform better than the current policy $\pi_\theta$. In practice $Q_{\text{ent}}^\mu(x,a)$ could be estimated via a single trajectory $Q_{\text{ent}}^{\mu}(x,a) \approx \tilde{R}^\mu(x,a) \coloneqq r_0 + \sum_{t\geq 1}^\infty \gamma^t (r_t + c \mathcal{H}^\mu(x_t))$. Though in Eqn.(\ref{eq:sil-q}) one could plug in $\hat{R}^\mu(x,a)$ in place of $Q^\mu(x,a)$  \citep{oh2018self,guo2019efficient}, this introduces bias due to the double-sample issue \citep{baird1995residual}, especially when $\hat{R}^\mu(x,a)$ has high variance either due to the dynamics or a stochastic policy.

\paragraph{\gls{SIL} with stochastic actor-critic.}
\gls{SIL} further focuses on actor-critic algorithms where the Q-function is parameterized by a value-function and a stochastic policy $Q_\theta(x,a) \coloneqq V_\theta(x) + c \log \pi_\theta(a\mid x)$. Taking gradients of the loss in Eqn.(\ref{eq:sil-q}) with respect to $\theta$ yields the following loss function of the value-function and policy. The full \gls{SIL} loss is $L_{\text{sil}}(\theta) = L_{\text{value}}(\theta) + L_{\text{policy}}(\theta)$.
\begin{align}
    L_{\text{value}}(\theta) = \frac{1}{2} ([\hat{R}^\mu(x,a) - V_\theta(x)]_+)^2, L_{\text{policy}}(\theta) = -\log \pi_\theta(a\mid x) [\tilde{R}^\mu(x,a) - V_\theta(x)]_+. \label{eq:sil-ac}
\end{align}

\section{Generalized Self-Imitation Learning}
\label{sec:sil}

\subsection{Generalized Lower Bounds for Optimal Q-functions}

To generalize the formulation of \gls{SIL}, we seek to provide generalized lower bounds for the optimal Q-functions. Practical lower bounds should possess several desiderata: \textbf{(P.1)} they could be estimated using off-policy partial trajectories; \textbf{(P.2)} they could bootstrap from learned Q-functions. 

In standard actor-critic algorithms, partial trajectories are generated via behavior policy $\mu$ (for example, see \citep{mnih2016,schulman2017,espeholt2018impala}), and the algorithm maintains an estimate of Q-functions for the current policy $\pi$. The following theorem states a general lower bound  for the max-entropy optimal Q-function  $Q_{\text{ent}}^{\pi_{\text{ent}}^\ast}$. Additional results on generalized lower bounds of the optimal value function $V^{\pi^\ast}$ could be similarly derived, and we leave its details in Theorem \ref{thm:v-bound}  in Appendix \ref{appendix:theoretical}.
\begin{theorem} 
\label{thm:maxent-gen-lowerbound}
(proof in Appendix \ref{appendix:lowerbound})
Let $\pi_{\text{ent}}^\ast$ be the optimal policy and $Q_{\text{ent}}^{\pi_{\text{ent}}^\ast}$ its Q-function under maximum entropy \gls{RL} formulation. Given a partial trajectory $(x_t,a_t)_{t=0}^{n}$, the following inequality holds for any $n$,
\begin{align}
    Q_{\text{ent}}^{\pi_{\text{ent}}^\ast}(x_0,a_0) \geq L_{\text{ent}}^{\pi,\mu,n}(x_0,a_0) \coloneqq \mathbb{E}_\mu[r_0 + \gamma c \mathcal{H}^\mu(x_1) + \sum_{t=1}^{n-1} \gamma^t (r_t + c \mathcal{H}^\mu(x_{t+1})) + \gamma^n Q_{\text{ent}}^\pi(x_n,a_n)]
    \label{eq:maxentrl-nstep-lowerbound}
\end{align}
\end{theorem}

By letting $c=0$, we derive a generalized lower bound for the standard optimal Q-function $Q^{\pi^\ast}$
\begin{lemma} 
\label{lemma:gen-lowerbound}
Let $\pi^\ast$ be the optimal policy and $Q^{\pi^\ast}$ its Q-function under standard \gls{RL}. Given a partial trajectory $(x_t,a_t)_{t=0}^{n}$, the following inequality holds for any $n$,
\begin{align}
    Q^{\pi^\ast}(x_0,a_0) \geq L^{\pi,\mu,n}(x_0,a_0) \coloneqq \mathbb{E}_\mu[\sum_{t=0}^{n-1} \gamma^t r_t + \gamma^n Q^\pi(x_n,a_n)].
    \label{eq:rl-nstep-lowerbound}
\end{align}
\end{lemma}

We see that $n$-step lower bounds $L_{\text{ent}}^{\pi,\mu,n}$ satisfy both desiderata (\textbf{P.1})(\textbf{P.2}): $L_{\text{ent}}^{\pi,\mu,n}$ could be estimated on a single trajectory and bootstraps from learned Q-function $Q_\theta(x,a) \approx Q^\pi(x,a)$. When $n\rightarrow \infty$, $L_{\text{ent}}^{\pi,\mu,n} \rightarrow Q^\mu$ and we arrive at the lower bound employed by the original \gls{SIL} \citep{oh2018self}.  The original \gls{SIL} does not satisfy (\textbf{P.1})(\textbf{P.2}): the estimate of $Q^\mu$ requires full trajectories from finished episodes and does not bootstrap from learned Q-functions. In addition, because the lower bound $L_{\text{ent}}^{\pi,\mu}(x,a)$ bootstraps Q-functions at a finite step $n$, we expect it to partially mitigate the double-sample bias of $\hat{R}^\mu(x,a)$. Also, as the policy $\pi$ improves over time, the Q-function $Q^\pi(x,a)$ increases and the bound $L^{\pi,\mu,n}$ improves as well. On the contrary, the standard \gls{SIL} does not enjoy such advantages.

\subsection{Generalized Self-Imitation Learning}

\paragraph{Generalized \gls{SIL} with stochastic actor-critic.}
We describe the generalized \gls{SIL} for actor-critic algorithms. As developed in Section 2.3, such algorithms maintain a parameterized \emph{stochastic} policy $\pi_\theta(a\mid x)$ and value-function $V_\theta(x)$. Let $\hat{L}_{\text{ent}}^{\pi,\mu,n}(x,a)$ denote the sample estimate of the $n$-step lower bound, the loss functions are
\begin{align}
    L^{(n)}_{\text{value}}(\theta) = \frac{1}{2} ([\hat{L}_{\text{ent}}^{\pi,\mu,n}(x,a) - V_\theta(x)]_+)^2, L^{(n)}_{\text{policy}}(\theta) = -\log \pi_\theta(a\mid x) [\hat{L}_{\text{ent}}^{\pi,\mu,n}(x,a) - V_\theta(x)]_+. \label{eq:gen-sil-ac}
\end{align}

Note that the loss functions in Eqn.(\ref{eq:gen-sil-ac}) introduce updates very similar to A2C \citep{mnih2016}. Indeed, when removing the threshold function $[x]_+$ and setting the data distribution to be on-policy $\mu=\pi$, we recover the $n$-step A2C objective. 

\paragraph{Generalized \gls{SIL} with deterministic actor-critic.}
For continuous control, \gls{TD}-learning and deterministic policy gradients have proven highly sample efficient and high-performing \citep{lillicrap2015continuous,fujimoto2018addressing,haarnoja2018soft}. By construction, the generalized $n$-step lower bounds $L_{
\text{ent}}^{\pi,\mu,n}$ adopts $n$-step \gls{TD}-learning and should naturally benefit the aforementioned algorithms. Such algorithms maintain a parameterized Q-function $Q_\theta(x,a)$, which could be directly updated via the following loss
\begin{align}
    L_{\text{qvalue}}^{(n)}(\theta) = \frac{1}{2} ([\hat{L}_{\text{ent}}^{\pi,\mu,n}(x,a) - Q_\theta(x,a)]_+)^2.\label{eq:gen-sil-det-ac}
\end{align}


Interestingly, note that the above update Eqn.(\ref{eq:gen-sil-det-ac}) is similar to $n$-step Q-learning update \citep{hessel2018rainbow,barth2018distributed} up to the threshold function $[x]_+$. In Section \ref{sec:tradeoff}, we will discuss their formal connections in details. 

\paragraph{Prioritized experience replay.} Prior work on prioritized experience replay \citep{schaul2016,horgan2018distributed} proposed to sample tuples $(x_t,a_t,r_t)$ from replay buffer $\mathcal{D}$ with probability proportional to Bellman errors. We provide a straightforward extension by sampling proportional to the lower bound loss $[\hat{L}_{\text{ent}}^{\pi,\mu,n}(x,a) - Q_\theta(x,a)]_+$. This reduces to the sampling scheme in \gls{SIL} \citep{oh2018self} when letting $n\rightarrow \infty$.

\section{Trade-offs with Lower Bound Q-learning}
\label{sec:tradeoff}
When applying \gls{SIL} in practice, its induced loss functions are optimized jointly with the base loss functions \citep{oh2018self}: in the case of stochastic actor-critic, the full loss function is $L(\theta) \coloneqq L_{\text{ac}}(\theta) + L_\text{sil}(\theta)$, where $L_{\text{ac}}(\theta)$ is the original actor-critic loss function \citep{mnih2016}. The parameter is then updated via the gradient descent step $\theta = \theta - \nabla_\theta L(\theta)$. This makes it difficult to analyze the behavior of \gls{SIL} beyond the plain motivation of Q-function lower bounds. Though a  comprehensive analysis of \gls{SIL} might be elusive due to its empirical nature, we formalize the lower bound arguments via \gls{RL} operators and draw connections with  $n$-step Q-learning. Below, we present results for standard \gls{RL}.

\subsection{Operators for Generalized Lower Bound Q-learning}
First, we formalize the mathematical operator of \gls{SIL}. Let $Q \in \mathbb{R}^{|\mathcal{X}|\times|\mathcal{A}|}$ be a vector-valued Q-function. Given some behavior policy $\mu$, define the operator $\mathcal{T_{\text{sil}}} Q(x,a) \coloneqq Q(x,a) + [Q^\mu(x,a)-Q(x,a)]_+$ where $[x]_+ \coloneqq \max(x,0)$. This operator captures the defining feature of the practical lower bound Q-learning \citep{oh2018self}, where the Q-function $Q(x,a)$ receives learning signals only when $Q^\mu(x,a) > Q(x,a)$. For generalized \gls{SIL}, we similarly define $\mathcal{T_{\text{n,sil}}} Q(x,a) \coloneqq Q(x,a) + [(\mathcal{T}^\mu)^{n-1}\mathcal{T}^\pi Q(x,a) - Q(x,a) ]_+$, where $Q(x,a)$ is updated when $(\mathcal{T}^\mu)^{n-1} \mathcal{T}^\pi Q(x,a) > Q(x,a)$ as suggested in Eqn.(\ref{eq:gen-sil-ac},\ref{eq:gen-sil-det-ac}).

In practice, lower bound Q-learning is applied alongside other main iterative algorithms. Henceforth, we focus on policy iteration algorithms with the Bellman operator $\mathcal{T}^\pi$ along with its $n$-step variant  $(\mathcal{T}^{\mu})^{n-1} \mathcal{T}^\pi$. Though practical deep \gls{RL} implementations adopt additive loss functions, for theoretical analysis we consider a convex combination of these three operators, with coefficients $\alpha,\beta \in [0,1]$.
\begin{align}
    \mathcal{T_{\text{n,sil}}^{\alpha,\beta}} \coloneqq (1-\beta)\mathcal{T}^\pi + (1-\alpha)\beta \mathcal{T_{\text{n,sil}}} + \alpha\beta (\mathcal{T}^{\mu})^{n-1} \mathcal{T}^\pi
    \label{eq:general-operator}
\end{align}

\subsection{Properties of the operators}
\begin{theorem} 
\label{thm:sil-nstep-operator}
(proof in Appendix \ref{appendix:operator})
Let $\pi,\mu$ be target and behavior policy respectively. Then the following results hold:
\begin{itemize}[topsep=0pt,parsep=0pt,partopsep=0pt,leftmargin=*]
    \item \textbf{Contraction rate.} $\Gamma(\mathcal{T_{\text{n,sil}}^{\alpha,\beta}}) \leq (1-\beta)\gamma+(1-\alpha)\beta+\alpha\beta\gamma^n$. The operator is always contractive for $\alpha \in [0,1], \beta \in [0,1)$. When $\alpha > \frac{1-\gamma}{1-\gamma^n}$, we have for any $\beta \in (0,1)$, $\Gamma(\mathcal{T_{\text{n,sil}}^{\alpha,\beta}}) \leq \gamma^\prime < \gamma$ for some $\gamma^\prime$. 
    \item \textbf{Fixed point bias.} $\mathcal{T_{\text{n,sil}}^{\alpha,\beta}}$ has a unique fixed point $\tilde{Q}^{\alpha,\beta}$ for any $\alpha \in [0,1], \beta \in [0,1)$ such that $(1-\alpha)\beta < 1$. This fixed point satisfies the bounds $Q^{\eta\pi + (1-\eta)\mu^{n-1}\pi} \leq \tilde{Q}^{\alpha,\beta} \leq Q^{\pi^\ast}$, where $Q^{\eta\pi + (1-\eta)\mu^{n-1}\pi}$ is the unique fixed point of operator $\eta\mathcal{T}^\pi + (1-\eta) \mathcal{T}^{\mu^{n-1}\pi}$ with $\eta=\frac{1-\beta}{1-\beta+\alpha\beta}$.
\end{itemize}
\end{theorem}

To highlight the connections between uncorrected $n$-step and  \gls{SIL}, we discuss two special cases.
\begin{itemize}[topsep=0pt,parsep=0pt,partopsep=0pt,leftmargin=*]
    \item When $\alpha = 1$, $\mathcal{T_{\text{n,sil}}^{\alpha,\beta}}$ removes all the lower bound components and reduces to $(1-\beta) \mathcal{T}^\pi + \beta (\mathcal{T}^\mu)^{n-1}\mathcal{T}^\pi$. This recovers the trade-off results discussed in \citep{rowland2019adaptive}: when $\beta=1$, the operator becomes uncorrected $n$-step updates with the smallest possible contraction rate $\Gamma(\mathcal{T_{\text{n,sil}}^{\alpha,\beta}}) \leq \gamma^n$, but the fixed point $\tilde{Q}^{\alpha,\beta}$ is biased. In general, there is no lower bound on the fixed point so that its value could be arbitrary depending on both  $\pi$ and $\mu$. 
    \item When $\alpha \in (\frac{1-\gamma}{1-\gamma^n},1]$, $\mathcal{T_{\text{n,sil}}^{\alpha,\beta}}$ combines the lower bound operator. Importantly, unlike uncorrected $n$-step, now the fixed point is lower bounded $\tilde{Q}^{\alpha,\beta} \geq Q^{\eta\pi + (1-\eta) \mu^{n-1}\pi}$. Because such a fixed point bias is lower bounded, we call it \emph{positive bias}. Adjusting $\alpha$ creates a trade-off between  contraction rates and the positive fixed point bias.  In addition, the fixed point bias is safe in that it is upper bounded by the optimal Q-function, $\tilde{Q}^{\alpha,\beta} \leq Q^{\pi^\ast}$, which might be a desirable property in cases where over-estimation bias hurts the practical performance \citep{van2016deep,fujimoto2018addressing}. In Section \ref{section:exp}, we will see that such positive fixed point bias is beneficial to empirical performance, as similarly observed in \citep{he2016learning,oh2018self,guo2019efficient}. Though $\mathcal{T}_{\text{n,sil}}^{\alpha,\beta}$ does not contract as fast as the uncorrected $n$-step operator $(\mathcal{T}^\mu)^{n-1}\mathcal{T}^\pi$, it still achieves a bound on contraction rates strictly smaller than $\mathcal{T}^\pi$. As such, generalized \gls{SIL} also enjoys fast contractions relative to the baseline algorithm.
    
\end{itemize}  

\begin{wrapfigure}{r}{0.3\textwidth}
    \includegraphics[width=0.3\textwidth]{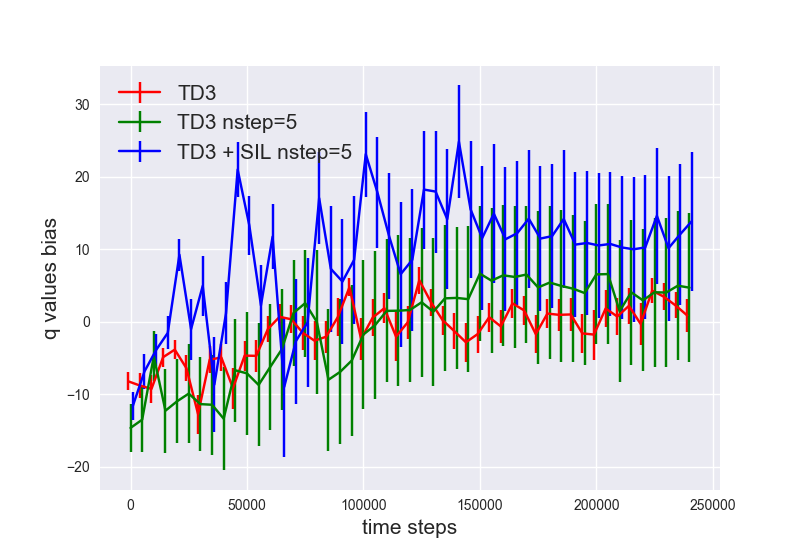}
  \caption{\small{Bias of Q-function networks with \gls{TD3} variants on the WalkerStand task.}
  \label{figure:qbias}
  }
\end{wrapfigure} 

\paragraph{Empirical evaluation of Q-function bias.} To validate the statements made in Theorem \ref{thm:sil-nstep-operator} on the bias of Q-functions, we test with \gls{TD3} for a an empirical evaluation \citep{fujimoto2018addressing}. At a given time in training, the bias at a pair $(x,a)$ is calculated as the difference between Q-function network prediction and an unbiased Monte-Carlo estimate of Q-function for the current policy $\pi$, i.e. $Q_\theta(x,a) - \hat{Q}^\pi(x,a)$\footnote{By definition, the bias should be the difference between the fixed point $\tilde{Q}$ and target $Q^\pi$. Since \gls{TD3} employs heavy replay during training, we expect the Q-function to be close to the fixed point $Q_\theta\approx \tilde{Q}$. Because both the dynamics and policy are deterministic, an one-sample estimate of Q-function is accurate enough to approximate the true Q-function $\hat{Q}^\pi=Q^\pi$. Hence here the bias is approximated by $Q_\theta-\hat{Q}^\pi$.}. Figure \ref{figure:qbias} shows the $\text{mean}\pm 0.5\text{std}$ of such bias over time, with $\text{mean}$ and $\text{std}$ computed over visited state-action pairs under $\pi$.
In general, the bias of \gls{TD3} is small, which is compatible to observations made in \citep{fujimoto2018addressing}. The bias of \gls{TD3} with uncorrected $n=5$-step spreads over a wider range near zero, indicating significant non-zero bias on both sides. For \gls{TD3} with generalized \gls{SIL} $n=5$, the bias is also spread out but the mean bias is significantly greater than zero. This implies that \gls{SIL} generally induces a positive bias in the fixed point. In summary, these observations confirm that neural network based Q-functions $Q_\theta(x,a)$ display similar biases introduced by the corresponding exact operators.






\section{Experiments}
\label{section:exp}
We seek to address the following questions in the experiments: \textbf{(1)} Does generalized \gls{SIL} entail performance gains on both deterministic and stochastic actor-critic algorithms? \textbf{(2)} How do the design choices (e.g. hyper-parameters, prioritized replay) of generalized \gls{SIL} impact its performance?

\paragraph{Benchmark tasks.} For benchmark tasks, we focus on state-based continuous control. In order to assess the strengths of different algorithmic variants, we consider similar tasks \emph{Walker}, \emph{Cheetah} and \emph{Ant} with different simulation backends from OpenAI gym \citep{brockman2016}, DeepMind Control Suite \citep{tassa2018deepmind} and Bullet Physics Engine \citep{coumans2010bullet}. These backends differ in many aspects, e.g. dimensions of observation and action space, transition dynamics and reward functions. With such a wide range of varieties, we seek to validate algorithmic gains with sufficient robustness to varying domains. There are a total of $8$ distinct simulated control tasks, with details in Appendix \ref{appendix:exp}.

\subsection{Deterministic actor-critic}

\paragraph{Baselines.} We choose \gls{TD3} \citep{fujimoto2018addressing} as the baseline algorithm which employs a deterministic actor $\pi_\phi(x)$. \gls{TD3} builds on \gls{DDPG} \citep{lillicrap2015continuous} and alleviates the over-estimation bias in \gls{DDPG} via delayed updates and double critics similar to double Q-learning \citep{hasselt2010double,van2016deep}. Through a comparison of \gls{DDPG} and \gls{TD3} combined with generalized \gls{SIL}, we will see that over-estimation bias makes the advantages through lower bound Q-learning much less significant. To incorporate generalized \gls{SIL} into \gls{TD3}, we adopt an additive loss function: let $L_{\text{TD3}^{(n)}}(\theta)$ be the $n$-step \gls{TD3} loss function and $L_{\text{sil}}^{(m)}(\theta)$ be the $m$-step generalized \gls{SIL} loss. The full loss is $L(\theta)\coloneqq L_{\text{TD3}}^{(n)}(\theta) + \eta L_{\text{sil}}^{(m)}(\theta)$ with some $\eta\geq 0$. We will use this general loss template to describe algorithmic variants for comparison below.

\paragraph{Return-based \gls{SIL} for \gls{TD3}.} A straightforward extension of \gls{SIL} \citep{oh2018self} and optimality tightening \citep{he2016learning} to deterministic actor-critic algorithms, is to estimate the return $\hat{R}^\mu(x_t,a_t) \coloneqq \sum_{t^\prime \geq t} \gamma^{t-t^\prime} r_{t^\prime}$ on a single trajectory $(x_t,a_t,r_t)_{t=0}^\infty$ and minimize the lower bound objective $([\hat{R}^\mu(x,a) - Q_\theta(x,a)]_+)^2$. Note that since both the policy and the transition is deterministic (for benchmarks listed above), the one-sample estimate of returns is exact in that $\hat{R}^\mu(x,a) \equiv R^\mu(x,a) \equiv Q^\mu(x,a)$. In this case, return-based \gls{SIL} is exactly equivalent to generalized \gls{SIL} with $n\rightarrow\infty$.

\paragraph{Evaluations.} We provide evaluations on a few standard benchmark tasks in Figure \ref{figure:td3} as well as their variants with delayed rewards. To facilitate the credit assignment of the training performance to various components of the generalized \gls{SIL}, we compare with a few algorithmic variants: $1$-step \gls{TD3} ($n=1,\eta=0$); $5$-step \gls{TD3} ($n=5,\eta=0$); \gls{TD3} with $5$-step generalized \gls{SIL} ($n=1,\eta=0.1,m=5$); \gls{TD3} with return-based \gls{SIL} ($n=1,\eta=0.1,m=\infty$). Importantly, note that the weighting coefficient is fixed $\eta=0.1$ for all cases of generalized \gls{SIL}. The training results of selected algorithms are shown in Figure \ref{figure:td3}. We show the final performance of all baselines in Table \ref{table:summary} in Appendix \ref{appendix:exp}. 



\begin{figure}[h]
\centering
\subfigure[\textbf{DMWalkerRun}]{\includegraphics[width=.24\linewidth]{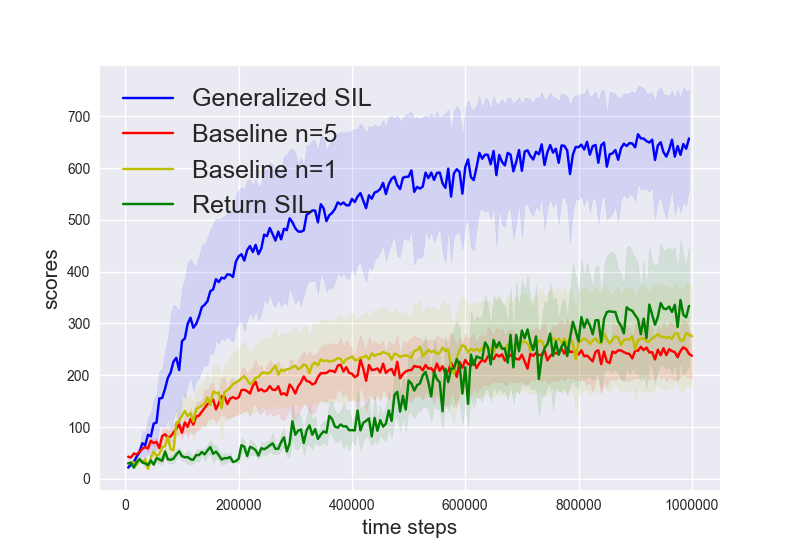}}
\subfigure[\textbf{DMWalkerStand}]{\includegraphics[width=.24\linewidth]{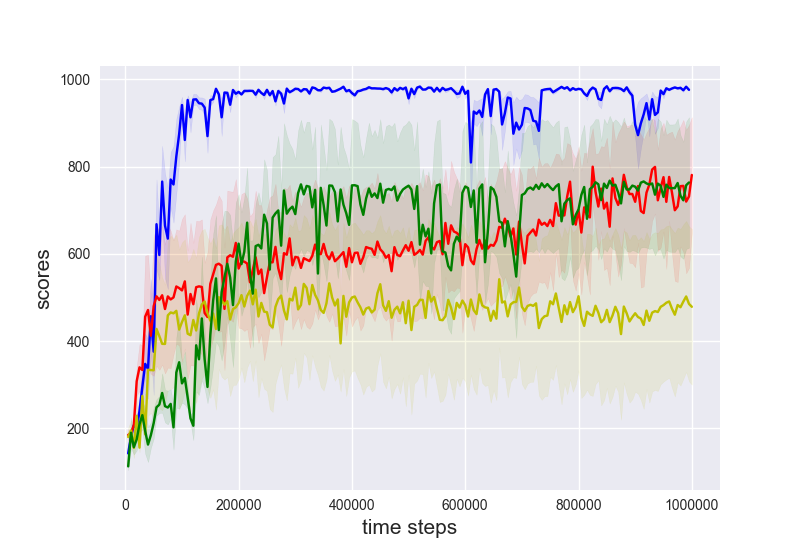}}
\subfigure[\textbf{DMWalkerWalk}]{\includegraphics[width=.24\linewidth]{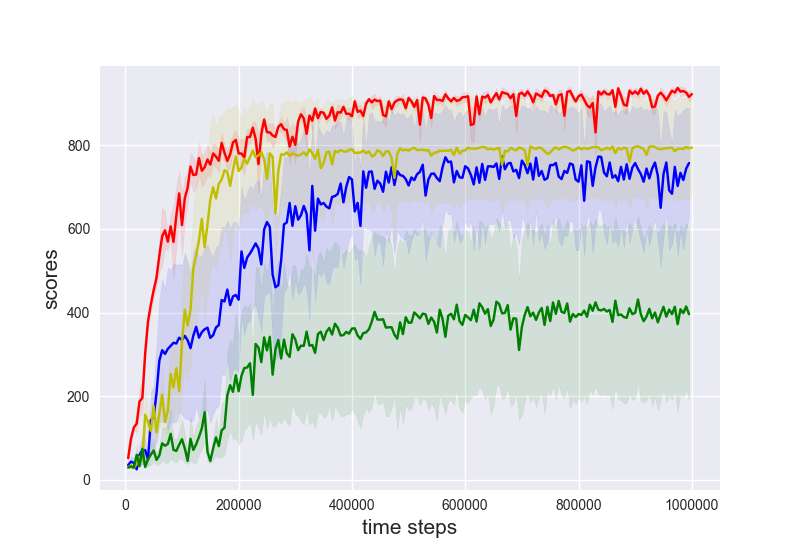}}
\subfigure[\textbf{DMCheetahRun}]{\includegraphics[width=.24\linewidth]{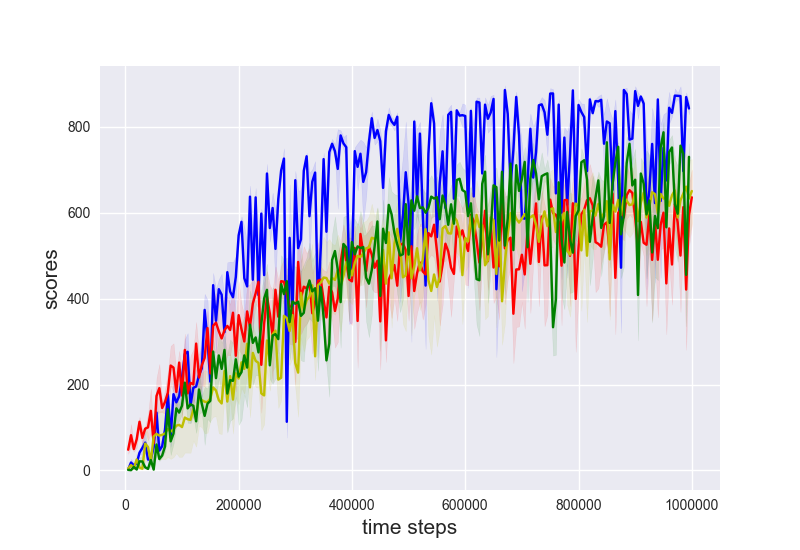}}
\subfigure[\textbf{Ant}]{\includegraphics[width=.24\linewidth]{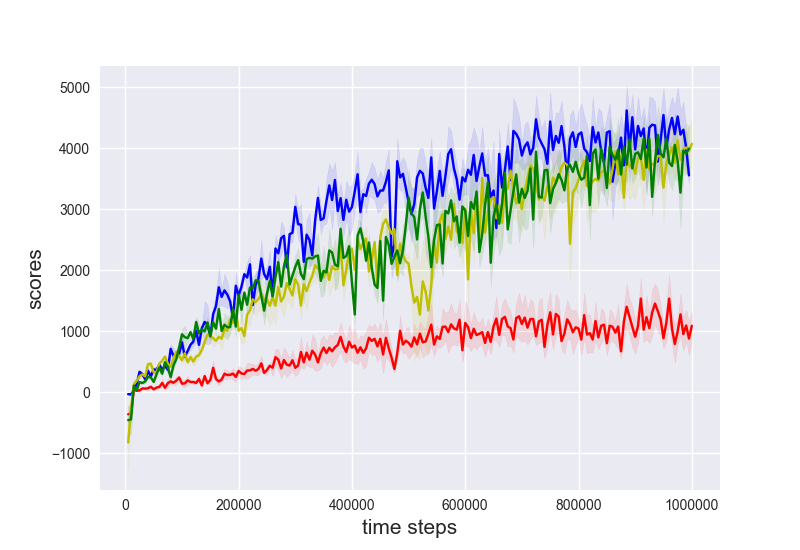}}
\subfigure[\textbf{HalfCheetah}]{\includegraphics[width=.24\linewidth]{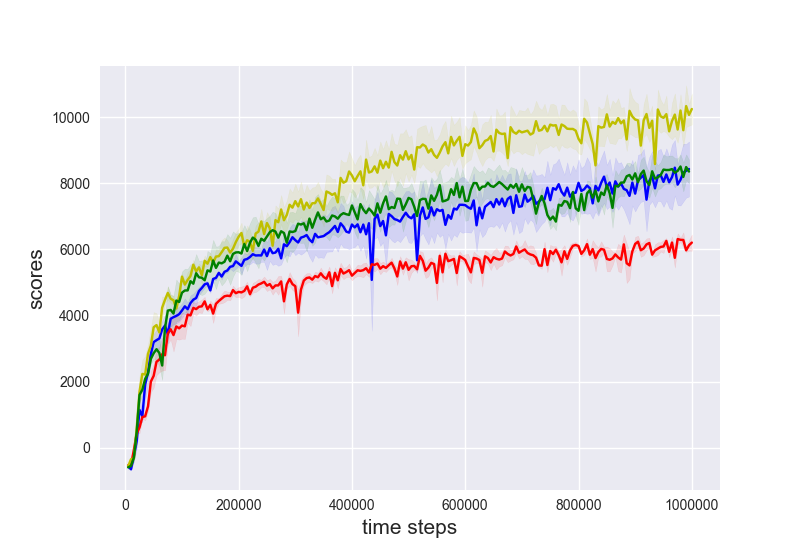}}
\subfigure[\textbf{Ant(B)}]{\includegraphics[width=.24\linewidth]{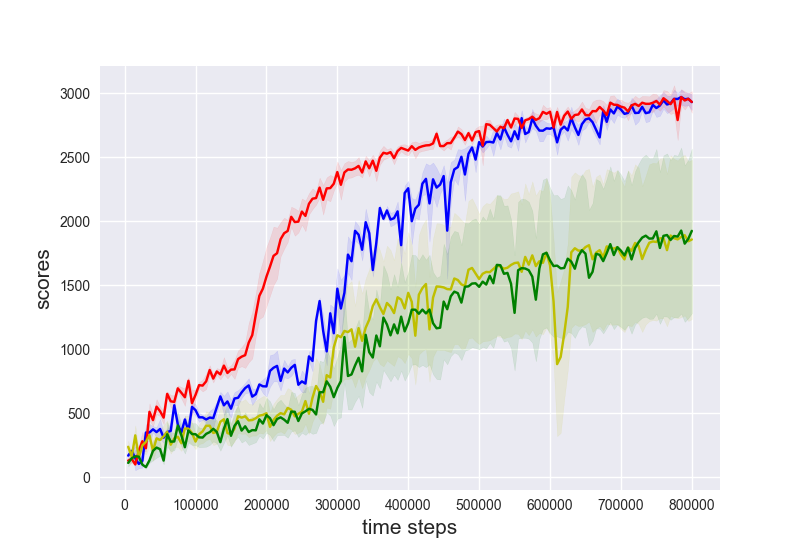}}
\subfigure[\textbf{HalfCheetah(B)}]{\includegraphics[width=.24\linewidth]{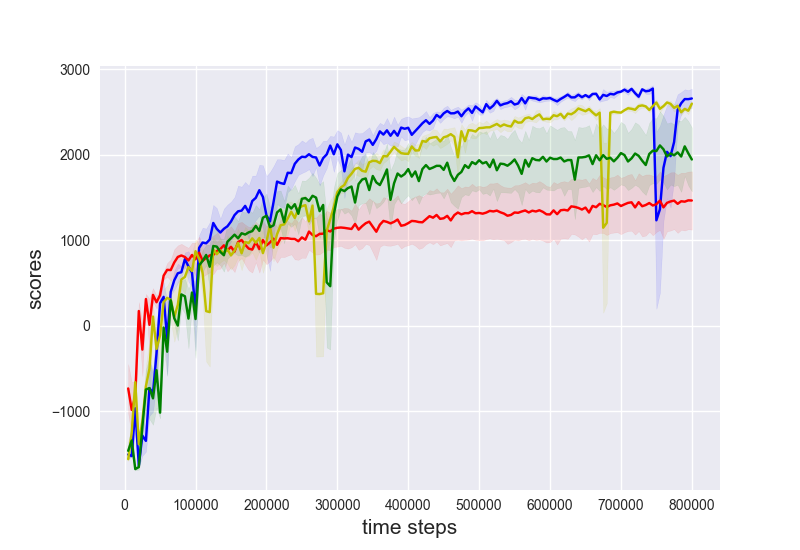}}
\caption{\small{Standard evaluations on $8$ benchmark tasks. Different colors represent different algorithmic variants. Each curve shows the $\text{mean}\pm 0.5 \text{std}$ of evaluation performance during training, averaged across $3$ random seeds. The x-axis shows the time steps and the y-axis shows the cumulative returns. Observe that $5$-step generalized \gls{SIL} (blue) generally outperforms other baselines. Tasks with \emph{DM} are from DeepMind Control Suite, and tasks with \emph{(B)} are from Bullet.}}
\label{figure:td3}
\end{figure}

We make several observations: (1) For uncorrected $n$-step, the best $n$ is task dependent. However, $5$-step generalized \gls{SIL} consistently improves the performance over uncorrected $n$-step \gls{TD3} baselines; (2) \gls{SIL} losses generally accelerate the optimization. Indeed, both generalized \gls{SIL} and return-based \gls{SIL} generally performs better than pure \gls{TD3} algorithms; (3) The advantage of generalized \gls{SIL} is more than $n$-step bootstrap. Because $n$-step generalized \gls{SIL} is similar to $n$-step updates, it is reasonable to speculate that the performance gains of \gls{SIL} are partly attributed to $n$-step updates. By the significant advantages of generalized \gls{SIL} relative to $n$-step updates, we see that its performance gains also come from the lower bound techniques; (4) $n$-step \gls{SIL} with $n=5$ works the best. With $n=1$, \gls{SIL} does not benefit from bootstrapping partial trajectories with long horizons; with $n=\infty$, \gls{SIL} does not benefit from bootstrapped values at all. As discussed in Section \ref{sec:sil}, $n$-step bootstrap provides benefits in (i) variance reduction (replacing the discounted sum of rewards by a value function) and (ii) tightened bounds. In deterministic environment with deterministic policy, the advantage (ii) leads to most of the performance gains.

\subsection{Ablation study for deterministic actor-critic}

Please refer to Table \ref{table:summary} in Appendix \ref{appendix:exp} for a summary of ablation experiments over \gls{SIL} variants. Here, we focus on discussions of the ablation results.

\paragraph{Horizon parameter $n$.} In our experience, we find that $n=5$ works reasonably well though other close values might work as well. To clarify the extreme effect of $n$: at one extreme, $n=1$ and \gls{SIL} does not benefit from trajectory-based learning and generally underperforms $n=5$; when $n=\infty$, the return-based \gls{SIL} does not provide as significant speed up as $n=5$.

\paragraph{Prioritized experience replay.} In general, prioritized replay has two hyper-parameters: $\alpha$ for the degree of prioritized sampling and $\beta$ for the degree of corrections \citep{schaul2016}. For general \gls{SIL}, we adopt $\alpha=0.6,\beta=0.1$ as in \citep{oh2018self}. We also consider  variants where the tuples are sampled according to the priority but IS weights are not corrected ($\alpha=0.6,\beta=0.0$) and where there is no prioritized sampling ($\alpha=\beta=0.0$). The results are reported in Table \ref{table:priority}
in Appendix \ref{appendix:exp}. We observe that generalized \gls{SIL} works the best when both prioritized sampling and IS corrections are present.


\paragraph{Over-estimation bias.} Algorithms with over-estimation bias (e.g. \gls{DDPG}) does not benefit as much (e.g. \gls{TD3}) from the lower bound loss, as shown by additional results in Appendix \ref{appendix:exp}. We speculate that this is because by construction the Q-function network $Q_\theta(x,a)$ should be a close approximation to the Q-function $Q^{\pi}(x,a)$. In cases where over-estimation bias is severe, this assumption does not hold. As a result, the performance is potentially harmed instead of improved by the \emph{uncontrolled} positive bias \citep{van2016deep,fujimoto2018addressing}. This contrasts with the \emph{controlled} positive bias of \gls{SIL}, which improves the performance.

\subsection{Stochastic actor-critic}
\paragraph{Baselines.} For the stochastic actor-critic algorithm, we adopt \gls{PPO} \citep{schulman2017}. Unlike critic-based algorithms such as \gls{TD3}, \gls{PPO} estimates gradients using near on-policy samples.

\paragraph{Delayed reward environments.} Delayed reward environment tests algorithms' capability to tackle delayed feedback in the form of sparse rewards \citep{oh2018self}. In particular, a standard benchmark environment returns dense reward $r_t$ at each step $t$. Consider accumulating the reward over $d$ consecutive steps and return the sum at the end $k$ steps, i.e. $r_t^\prime=0$ if $ t \ \text{mod}\ k \neq 0$ and $r_t^\prime = \sum_{\tau=t-d+1}^t r_\tau$ if $t\ \text{mod}\ d =0$. 

\paragraph{Evaluations.} We compare three baselines: \gls{PPO}, \gls{PPO} with \gls{SIL} \citep{oh2018self} and \gls{PPO} with generalized \gls{SIL} with $n=5$-step. We train these variants on a set of OpenAI gym tasks with delayed rewards, where the delays are $k\in \{1,5,10,20\}$. Please refer to Appendix \ref{appendix:exp} for further details of the algorithms. The final performance of algorithms after training ($5\cdot 10^6$ steps for HalfCheetah and $10^7$ for the others) are shown in Figure \ref{figure:ppo}. We make several observations: \textbf{(1)} The performance of \gls{PPO} is generally inferior to its generalized \gls{SIL} or \gls{SIL} extensions. This implies the necessity of carrying out \gls{SIL} in general, as observed in \citep{oh2018self}; \textbf{(2)} The performance of generalized \gls{SIL} with $n=5$ differ depending on the tasks. \gls{SIL} works significantly better with Ant, while generalized \gls{SIL} works better with Humanoid. Since \gls{SIL} is a special case for $n\rightarrow \infty$, this implies the potential benefits of adapting $n$ for each task.

\begin{figure}[h]
\centering
\subfigure[\textbf{HalfCheetah}]{\includegraphics[width=.24\linewidth]{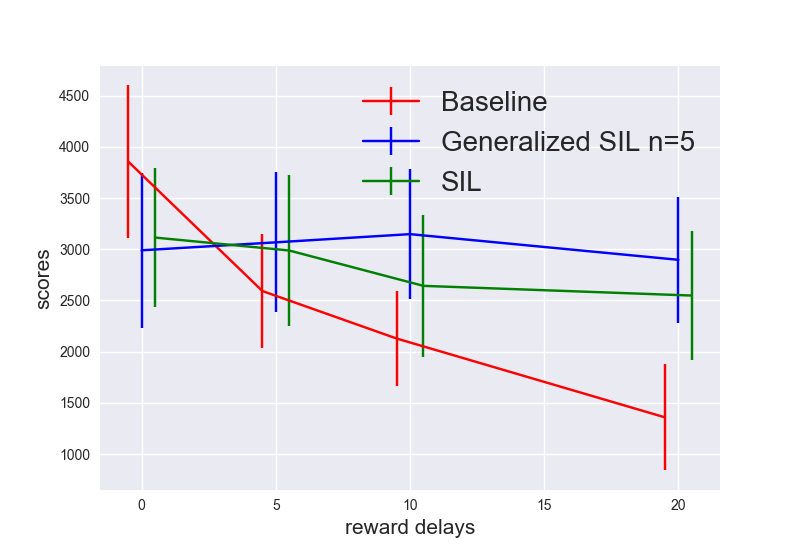}}
\subfigure[\textbf{Ant}]{\includegraphics[width=.24\linewidth]{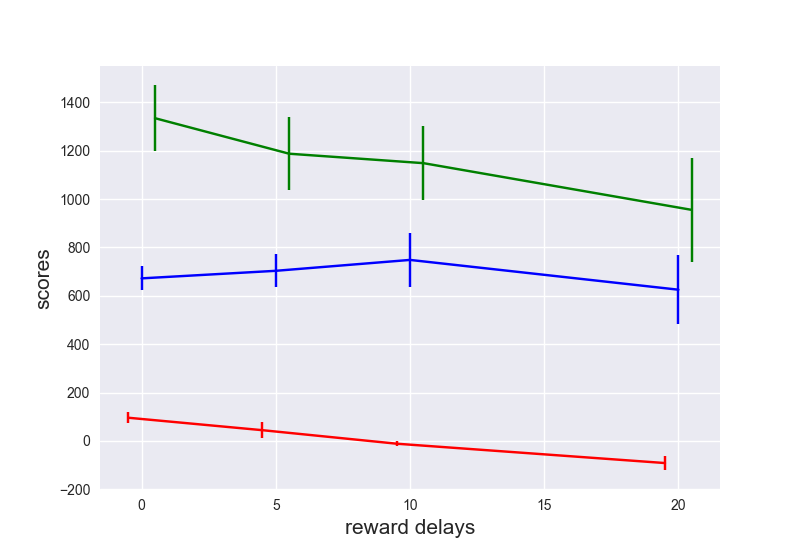}}
\subfigure[\textbf{Walke2d}]{\includegraphics[width=.24\linewidth]{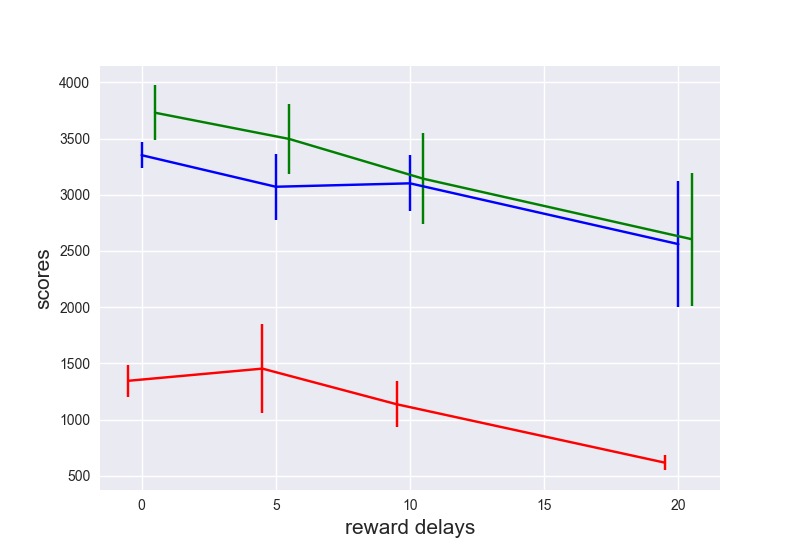}}
\subfigure[\textbf{Humanoid}]{\includegraphics[width=.24\linewidth]{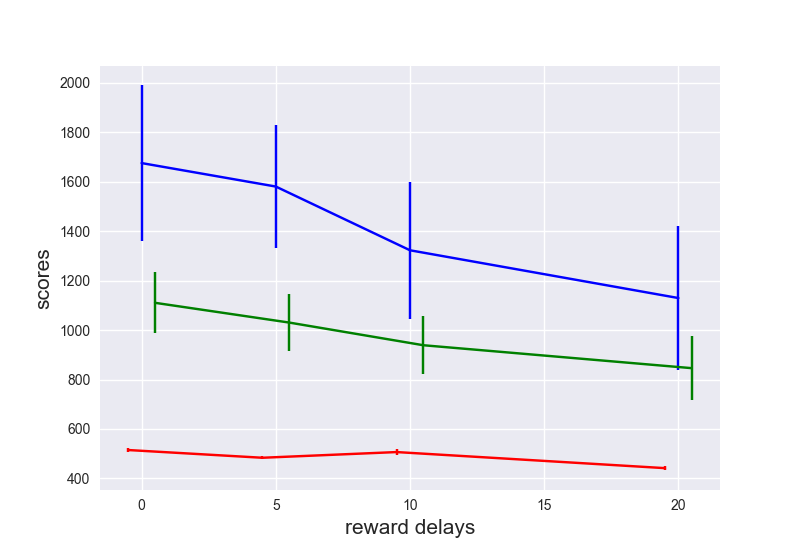}}
\caption{\small{Standard evaluations on $4$ benchmark OpenAI gym tasks. Different colors represent different algorithmic variants. Each curve shows the $\text{mean}\pm 0.5 \text{std}$ of evaluation performance at the end of training, averaged across $5$ random seeds. The x-axis shows the delayed time steps for rewards and the y-axis shows the cumulative returns. The ticks $\{1,5,10,20\}$ show the delays and the x-axis of the plotted data is slightly shifted for better visualization.}}
\label{figure:ppo}
\end{figure}

\section{Further Discussions on Related Work}

\paragraph{Over-estimation bias in Q-learning.} Q-learning and \gls{TD}-learning are popular algorithms for \gls{RL} \citep{watkins1992q,bertsekas1995neuro}. Due to the max operator, sampled updates of Q-learning naturally incur over-estimation bias, which potentially leads to unstable learning. To mitigate the bias, prior work has considered Double Q-learning\cite{hasselt2010double,van2016deep}, explicit bias correction \citep{lee2013bias}, linear combination between Double Q-learning and Q-learning \citep{zhang2017weighted}, bootstrapping from past predictions \citep{anschel2017averaged} and using an ensemble of Q-functions \citep{lan2020maxmin}. Similar ideas have been applied to actor-critic algorithms \citep{fujimoto2018addressing}. While it is conventionally believed that over-estimation bias is hurtful to the performance, \citep{lan2020maxmin} provides concrete examples where estimation bias (under- or over-estimation) could accelerate learning. In practice, for certain \gls{RL} environments where rewards are sparse, it is desirable to introduce positive bias to encourage exploration \citep{oh2018self}. 

\paragraph{Learning from off-policy data.} Off-policy learning is crucial for modern \gls{RL} algorithms \citep{sutton1998,szepesvari2010algorithms}. At the core of many off-policy learning algorithms \citep{precup2001off,dudik2014doubly,thomas2016data,munos2016safe,mahmood2017multi,farajtabar2018more}, importance sampling (\gls{IS}) corrects for the distributional mismatch between behavior $\pi$ and target policy $\mu$, generating \emph{unbiased} updates. Despite the theoretical foundations, \gls{IS}-based algorithms often underperform empirically motivated algorithms such as $n$-step updates \citep{hessel2018rainbow,barth2018distributed,kapturowski2018recurrent}. In general, uncorrected $n$-step algorithms could be interpreted as trading-off fast contractions with fixed point bias \citep{rowland2019adaptive}, which seems to have a significant practical effect. In addition to potentially better performance, uncorrected $n$-step updates also do not require e.g. $\mu(a\mid x)$. This entails learning with truly arbitrary off-policy data. Built on top of $n$-step updates, we propose generalized $n$-step \gls{SIL} which intentionally introduces a positive bias into the fixed point, effectively filtering out behavior data with poor performance. This idea of learning from good-performing off-policy data is rooted in algorithmic paradigms such as behavior cloning \citep{ross2011reduction}, inverse \gls{RL} \citep{abbeel2004apprenticeship}, and more recently instantiated by e.g., episodic control \citep{blundell2016model,pritzel2017neural} lower bound Q-learning \citep{he2016learning} and \gls{SIL} \citep{oh2018self,gangwani2018learning,guo2019efficient}.

\section{Conclusion}
We have proposed generalized $n$-step lower bound Q-learning, a strict generalization of return-based lower bound Q-learning and the corresponding self-imitation learning algorithm \citep{oh2018self}. We have drawn close connections between $n$-step lower bound Q-learning and uncorrected $n$-step updates: both techniques achieve  performance gains by invoking a trade-off between contraction rates and fixed point bias of the evaluation operators. Empirically, we observe that the positive bias induced by lower bound Q-learning provides more consistent improvements than arbitrary $n$-step bias. It is of interest to study in general what bias could be beneficial to policy optimization, and how to exploit such bias in practical \gls{RL} algorithms.

\section{Broader Impact}

Algorithms which learn from off-policy samples are critical for the applications of \gls{RL} to more impactful real life domains such as autonomous driving and health care. Our work provides insights into \gls{SIL}, and its close connections to popular off-policy learning techniques such as $n$-step Q-learning. We believe our work entails a positive step towards better understanding of efficient off-policy \gls{RL} algorithms, which paves the way for future research into important applications. 

\section{Acknowledgements}
The author thanks Mark Rowland and Tadashi Kozuno for insightful discussions about this project.

\newpage
\bibliographystyle{unsrt}
\bibliography{your_bib_file}

\newpage
\appendix

\section{Proof of Theorem \ref{thm:maxent-gen-lowerbound}}
\label{appendix:lowerbound}
Recall that under maximum entropy RL, the Q-function is defined as $Q_{\text{ent}}^\pi(x_0,a_0) \coloneqq \mathbb{E}_\pi[r_0 + \sum_{t=1}^\infty \gamma^t (r_t + c \mathcal{H}^\pi(x_t))]$ where $\mathcal{H}^\mu(x_t)$ is the entropy of the distribution $\pi^\mu(\cdot\mid x_t)$. The Bellman equation for Q-function is naturally
\begin{align*}
    Q^\pi(x_0,a_0) = \mathbb{E}_\pi[r_0 + \gamma c \mathcal{H}^\pi(x_1) + \gamma Q^\pi(x_1,a_1)].
\end{align*}
Let the optimal policy be $\pi_{\text{ent}}^\ast$. The relationship between the optimal policy and its Q-function is $\pi_{\text{ent}}^\ast(a\mid x) \propto \exp(Q_{\text{ent}}^{\pi_{\text{ent}}^\ast}(x,a) / c)$. We seek to establish $Q_{\text{ent}}^{\pi_{\text{ent}}^\ast}(x_0,a_0) \geq \mathbb{E}_\mu[r_0 + \gamma c \mathcal{H}^\mu(x_1) + \sum_{t=1}^{T-1} \gamma^t (r_t + c \mathcal{H}^\mu(x_{t+1})) + \gamma^T Q_{\text{ent}}^\pi(x_T,a_T)]$ for any policy $\mu,\pi$.

We prove the results using induction. For the base case $T=1$, \begin{align}
    Q_{\text{ent}}^{\pi_{\text{ent}}^\ast}(x_0,a_0) &= \mathbb{E}_{\pi_{\text{ent}}^\ast}[r_0 + \gamma c \mathcal{H}^{\pi_{\text{ent}^\ast}}(x_1) + \gamma Q_{\text{ent}}^{\pi_{\text{ent}}^\ast}(x_1,a_1)] \nonumber \\
    &= \mathbb{E}_{x_1 \sim p(\cdot\mid x_0,a_0)}\big[r_0 + \gamma c \mathcal{H}^{\pi_{\text{ent}^\ast}}(x_1) + \gamma \mathbb{E}_{\pi_{\text{ent}}^\ast}[ Q_{\text{ent}}^{\pi_{\text{ent}}^\ast}(x_1,a_1)]\big] \nonumber \\
    &\geq \mathbb{E}_{x_1 \sim p(\cdot\mid x_0,a_0)}\big[r_0 + \gamma c \mathcal{H}^\mu(x_1) + \gamma \mathbb{E}_{\mu}[ Q_{\text{ent}}^{\pi_{\text{ent}}^\ast}(x_1,a_1)]\big] \nonumber \\
    &\geq \mathbb{E}_{x_1 \sim p(\cdot\mid x_0,a_0)}\big[r_0 + \gamma c \mathcal{H}^\mu(x_1) + \gamma \mathbb{E}_{\mu}[ Q_{\text{ent}}^{\pi}(x_1,a_1)]\big]. \nonumber \\
\end{align}
In the above, to make the derivations clear, we single out the reward $r_0$ and state $x_1 \sim p(\cdot\mid x_0,a_0)$, note that the distributions of these two quantities do not depend on the policy. The first inequality follows from the fact that $\pi_{\text{ent}}^\ast(\cdot\mid x) = \arg\max_{\pi} [c \mathcal{H}^\pi(x) + \mathbb{E}_{a\sim \pi(\cdot\mid x)} Q_{\text{ent}}^{\pi_{\text{ent}}}(x,a)]$. The second inequality follows from $Q_{\text{ent}}^{\pi_{\text{ent}}^\ast}(x,a) \geq Q_{\text{ent}}^\pi(x,a)$ for any policy $\pi$.

With the base case in place, assume that the result holds for $T\leq k-1$. Consider the case $T=k$
\begin{align}
    \mathbb{E}_\mu[r_0 + \gamma c \mathcal{H}^\mu(x_1) &+ \sum_{t=1}^{T-1} \gamma^t (r_t + c \mathcal{H}^\mu(x_{t+1})) + \gamma^T Q_{\text{ent}}^\pi(x_T,a_T)]  \nonumber \\
    &\leq \mathbb{E}_\mu \big[r_0 + \gamma c \mathcal{H}^\mu(x_1) + \gamma \mathbb{E}_\mu [Q_{\text{ent}}^{\pi_{\text{ent}}^\ast}(x_1,a_1)]\big] \nonumber \\
    &\leq Q_{\text{ent}}^{\pi_{\text{ent}}^\ast}(x_0,a_0), \nonumber
\end{align}

When $\pi=\mu$ we have the special case $\mathbb{E}_{\mu}[\sum_{t=0}^{\infty} \gamma^t r_t] \leq V^{\pi^\ast}(x_0)$, the lower bound which motivated the original lower-bound Q-learning based self-imitation learning \citep{oh2018self}.

\section{Proof of Theorem \ref{thm:sil-nstep-operator}}
\label{appendix:operator}
For notational simplicity, let $\mathcal{U} \coloneqq (\mathcal{T}^\mu)^{n-1} \mathcal{T}^\pi$ and let $\tilde{\mathcal{U}} Q(x,a) \coloneqq Q(x,a) + [UQ(x,a) - Q(x,a)]_+$. As a result, we could write $\mathcal{T_{\text{n,sil}}^{\alpha,\beta}} = (1-\beta)\mathcal{T}^\pi + (1-\alpha) \beta \tilde{\mathcal{U}} + \alpha\beta \mathcal{U}$. 

First, we prove the contraction properties of  $\mathcal{T_{\beta,\text{n,sil}}^{\mu}}$. Note that by construction $|\tilde{\mathcal{U}}Q_1(x,a) - \tilde{\mathcal{U}}Q_2(x,a)| \leq \max (|Q_1(x,a) - Q_2(x,a)|,|\mathcal{U}Q_1(x,a) - \mathcal{U}Q_2(x,a)|) \leq \;\|\; Q_1-Q_2\;\|\; _\infty$. Then through the triangle inequality, $\;\|\; \mathcal{T_{\text{n,sil}}^{\alpha,\beta}} Q_1 - \mathcal{T_{\text{n,sil}}^{\alpha,\beta}} Q_2\;\|\; _\infty \leq (1-\beta) \;\|\; \mathcal{T}^\pi Q_1 - \mathcal{T}^\pi Q_2\;\|\; _\infty + (1-\alpha)\beta \;\|\; \tilde{\mathcal{U}}Q_1 - \tilde{\mathcal{U}}Q_2\;\|\; _\infty + \alpha\beta \;\|\; \mathcal{U}Q_1 - \mathcal{U}Q_2\;\|\; _\infty \leq [(1-\beta)\gamma + (1-\alpha)\beta + \alpha\beta\gamma^n] \;\|\; Q_1 - Q_2\;\|\; _\infty$. This proves the upper bound on the contraction rates of $\mathcal{T_{\text{n,sil}}^{\alpha,\beta}}$.  Let $\eta(\alpha,\beta) = (1-\beta)\gamma + (1-\alpha)\beta + \alpha\beta\gamma^n$ and set $\eta(\alpha,\beta) < \gamma$, we deduce $\alpha > \frac{1-\gamma}{1-\gamma^n}$.

Next, we show properties of the fixed point $\tilde{Q}^{\alpha,\beta}$. This point uniquely exists because $\Gamma(\mathcal{T_{\text{n,sil}}^{\alpha,\beta}}) < 1$ if $(1-\alpha)\beta < 1$. From $\mathcal{T_{\text{n,sil}}^{\alpha,\beta}} \tilde{Q}^{\alpha,\beta} = \tilde{Q}^{\alpha,\beta}$, we could derive by rearranging terms $(1-\beta) (\mathcal{T}^\pi \tilde{Q}  - \tilde{Q}) +  \alpha\beta (\mathcal{U} \tilde{Q} - \tilde{Q}) = -(1-\alpha)\beta (\tilde{\mathcal{U}} \tilde{Q} - \tilde{Q}) \leq 0$. This further implies that $\mathcal{T}^\pi \tilde{Q} \leq \tilde{Q}$. Now let $\mathcal{T} \coloneqq \frac{(1-\beta)}{1-\beta+\alpha\beta} \mathcal{T}^\pi + \frac{\alpha\beta}{1-\beta+\alpha\beta} \mathcal{U}$. This simplifies to $\mathcal{T} \tilde{Q} - \tilde{Q} \leq 0$. By the monotonicity of $\mathcal{T}$, we see $Q^{t\pi + (1-t)\mu^{n-1}\pi} \geq \lim_{k\rightarrow} (\mathcal{T})^k \tilde{Q} = Q^\pi$ where $t = \frac{1-\beta}{1-\beta+\alpha\beta}$.

For the another set of inequalities, define $\tilde{H} Q \coloneqq  (1-\beta)\mathcal{T}^\ast + (1-\alpha)\beta\tilde{\mathcal{U}}Q + \alpha\beta (\mathcal{T}^\ast)^n$, where recall that $\mathcal{T}^\ast$ is the optimality Bellman operator. 

First, note $\tilde{H}$ has $Q^{\pi^\ast}$ as its unique fixed point. To see why, let $\tilde{Q}$ be a generic fixed point of $\tilde{H}$ such that $\tilde{H}\tilde{Q} = \tilde{Q}$. By rearranging terms, it follows that $(1-\beta) (\mathcal{T}^\ast \tilde{Q}  - \tilde{Q}) +  \alpha\beta ((\mathcal{T}^\ast)^n \tilde{Q} - \tilde{Q}) = -(1-\alpha)\beta (\tilde{\mathcal{U}} \tilde{Q} - \tilde{Q}) \leq 0$. However, by construction $(\mathcal{T}^\ast)^i Q \geq Q, \forall i\geq 1, \forall Q$. This implies that $(1-\beta) (\mathcal{T}^\ast \tilde{Q}  - \tilde{Q}) +  \alpha\beta ((\mathcal{T}^\ast)^n \tilde{Q} - \tilde{Q}) \geq 0$. As a result, $(1-\beta) (\mathcal{T}^\ast \tilde{Q}  - \tilde{Q}) +  \alpha\beta ((\mathcal{T}^\ast)^n \tilde{Q} - \tilde{Q}) = 0$ and $\tilde{Q}$ is a fixed point of $t\mathcal{T}^\ast + (1-t) (\mathcal{T}^\ast)^n$. Since $t\mathcal{T}^\ast + (1-t) (\mathcal{T}^\ast)^n$ is strictly contractive as $\Gamma(t\mathcal{T}^\ast + (1-t) (\mathcal{T}^\ast)^n) \leq t\gamma + (1-t)\gamma^n \leq \gamma < 1$, its fixed point is unique. It is straightforward to deduce that $Q^{\pi^\ast}$ is a fixed point of $t\mathcal{T}^\ast + (1-t) (\mathcal{T}^\ast)^n$ and we conclude that the only possible fixed point of $\tilde{H}$ is $\tilde{Q} = Q^{\pi^\ast}$. Finally, recall that by construction $\tilde{H} Q \geq Q,\forall Q$. By monotonicity, $Q^{\pi^\ast} = \lim_{k\rightarrow\infty}(\tilde{H})^k \tilde{Q}^{\alpha,\beta} \geq \tilde{Q}^{\alpha,\beta}$.
In conclusion, we have shown $Q^{t\pi + (1-t)\mu^{n-1}\mu} \leq \tilde{Q}^{\alpha,\beta} \leq Q^{\pi^\ast}$.

\section{Additional theoretical results}
\label{appendix:theoretical}
\begin{theorem} 
\label{thm:v-bound}
Let $\pi^\ast$ be the optimal policy and $V^{\pi^\ast}$ its value function under standard \gls{RL} formulation. Given a partial trajectory $(x_t,a_t)_{t=0}^{n}$, the following inequality holds for any $n$,
\begin{align}
    V^{\pi^\ast}(x_0) \geq \mathbb{E}_\mu [\sum_{t=0}^{n-1} \gamma^t r_t + \gamma^n V^\pi(x_k)]
    \label{eq:v-lowerbound}
\end{align}
\end{theorem}
\begin{proof}
Let $\pi,\mu$ be any policy and $\pi^\ast$ the optimal policy. We seek to show $V^{\pi^\ast}(x_0) \geq \mathbb{E}_{\mu}[\sum_{t=0}^{T-1} \gamma^t r_t + \gamma^T V^\pi(x_T)]$ for any $T \geq 1$.

We prove the results using induction. For the base case $T=1$, $V^{\pi^\ast}(x_0) = \mathbb{E}_{\pi^\ast}[Q^{\pi^\ast}(x_0,a_0)] \geq \mathbb{E}_{\mu}[Q^{\pi^\ast}(x_0,a_0)] = \mathbb{E}_{\mu}[r_0 + \gamma V^{\pi^\ast}(x_1)] \geq \mathbb{E}_{\mu}[r_0 + \gamma V^{\pi}(x_1)] $, where the first inequality comes from the fact that $\pi^\ast(\cdot\mid x_0) = \arg\max_a Q^{\pi^\ast}(x_0,a) $. Now assume that the statement holds for any $T\leq k-1$, we proceed to the case $T=k$.
\begin{align}
    \mathbb{E}_{\mu}[\sum_{t=0}^{k-1} \gamma^t r_t + \gamma^k V^\pi(x_k)] &= \mathbb{E}_{\mu}\big[r_0 + \gamma \mathbb{E}_\mu [ \sum_{t=0}^{k-2} \gamma^t r_t + \gamma^{k-1} V^{\pi}(x_k)]\big] \nonumber \\
    &\leq \mathbb{E}_{\mu}[r_0 + \gamma V^{\pi^\ast}(x_1)] \leq V^{\pi^\ast}(x_0), \nonumber 
\end{align}
where the first inequality comes from the induction hypothesis and the second inequality follows naturally from the base case. This implies that $n$-step quantities of the form $V^{\pi^\ast}(x_0) \geq \mathbb{E}_{\mu}[\sum_{t=0}^{n-1} \gamma^t r_t + \gamma^n V^\pi(x_T)]$ are lower bounds of the optimal value function $V^{\pi^\ast}(x_0)$ for any $n\geq 1$. 
\end{proof}

\section{Experiment details}
\label{appendix:exp}

\paragraph{Implementation details.} The algorithmic baselines for deterministic actor-critic ( \gls{TD3} and \gls{DDPG}) are based on OpenAI Spinning Up \url{https://github.com/openai/spinningup} \citep{achiam2018openai}. The baselines for stochastic actor-critic is based on \gls{PPO} \citep{schulman2017} and \gls{SIL}$+$\gls{PPO}  \citep{oh2018self} are based on the author code base \url{https://github.com/junhyukoh/self-imitation-learning}. Throughout the experiments, all optimizations are carried out via Adam optimizer \citep{kingma2014adam}.

\paragraph{Architecture.} Deterministic actor-critic baselines, including \gls{TD3} and \gls{DDPG} share the same network architecture following \citep{achiam2018openai}. The Q-function network $Q_\theta(x,a)$ and policy $\pi_\phi(x)$ are both $2$-layer neural network with $h=300$ hidden units per layer, before the output layer. Hidden layers are interleaved with $\text{relu}(x)$ activation functions. For the policy $\pi_\phi(x)$, the output is stacked with a $\text{tanh}(x)$ function to ensure that the output action is in $[-1,1]$. All baselines are run with default hyper-parameters from the code base. 

Stochastic actor-critic baselines (e.g. \gls{PPO}) implement value function $V_\theta(x)$ and policy $\pi_\phi(a\mid x)$ both as $2$-layer neural network with $h=64$ hidden units per layer and $\text{tanh}$ activation. The stochastic policy $\pi_\phi(a\mid x)$ is a Gaussian $a\sim \mathcal{N}(\mu_\phi(x), \sigma^2)$ with state-dependent mean $\mu_\phi(x)$ and a global variance parameter $\sigma^2$. Other missing hyper-parameters take default values from the code base.

\subsection{Further implementation and hyper-parameter details}

\paragraph{Generalized \gls{SIL} for deterministic actor-critic.}
We adopt \gls{TD3} \citep{fujimoto2018addressing} as the baseline for deterministic actor-critic. \gls{TD3} maintains a Q-function network $Q_\theta(x,a)$ and a deterministic policy network $\pi_\theta(x)$ with parameter $\theta$. The \gls{SIL} subroutines adopt a prioritized experience replay buffer: the return-based \gls{SIL} samples tuples according to the priority $[R^\mu(x,a) - Q_\theta(x,a)]_+$ and minimizes the loss function $[R^\mu(x,a) - Q_\theta(x,a)]_+$; the generalized \gls{SIL} samples tuples according to the priority $[L^{\pi,\mu,n}(x,a) - Q_\theta(x,a)]_+$ and minimizes the loss function $[L^{\pi,\mu,n}(x,a) - Q_\theta(x,a)]_+$. The experience replay adopts the parameter $\alpha=0.6,\beta=0.1$ \citep{schaul2015prioritized}. Throughout the experiments, \gls{TD3}-based algorithms all employ $\alpha=10^{-3}$ for the network updates.

To calculate the update target $L^{\pi,\mu,n}(x_0,a_0) = \sum_{t=0}^{n-1} \gamma^t r_t + Q_{\theta^\prime}(x_n,\pi_{\theta^\prime}(x_n))$ with partial trajectory $(x_t,a_t,r_t)_{t=0}^n$ along with the target value network $Q_{\theta^\prime}(x,a)$ and policy network $\pi_{\theta^\prime}(x)$. The target network is slowly updated as $\theta^\prime = \tau \theta^\prime + (1-\tau) \theta$ where $\tau=0.995$ \citep{mnih2013}.

\paragraph{Generalized \gls{SIL} for stochastic actor-critic.} We adopt \gls{PPO} \citep{schulman2017} as the baseline algorithm and implement modifications on top of the \gls{SIL} author code base \url{https://github.com/junhyukoh/self-imitation-learning} as well as the original baseline code \url{https://github.com/openai/baselines} \citep{baselines}. All \gls{PPO} variants use the default learning rate $\alpha=3\cdot10^{-4}$ for both actor $\pi_\theta(a\mid x)$ and critic $V_\theta(x)$. The \gls{SIL} subroutines are implemented as a prioritized replay with $\alpha=0.6,\beta=0.1$. For other details of \gls{SIL} in \gls{PPO}, please refer to the \gls{SIL} paper \citep{oh2018self}.

The only difference between generalized \gls{SIL} and \gls{SIL} lies in the implementation of the prioritized replay. \gls{SIL} samples tuples according to the priority $[R^\mu(x,a) - V_\theta(x)]_+$ and minimize the \gls{SIL} loss function $([R^\mu(x,a) - V_\theta(x)]_+)^2$ for the value function, and $-\log \pi_\theta(a\mid x) [R^\mu(x,a) - V_\theta(x)]_+$ for the policy. Generalized \gls{SIL} samples tuples according to the priority $([L^{\pi,\mu,n}(x,a) - V_\theta(x)]_+)^2$, and minimize the loss $([L^{\pi,\mu,n}(x,a) - V_\theta(x)]_+)^2$ and $-\log \pi_\theta(a\mid x) [L^{\pi,\mu,n}(X,a) - V_\theta(x)]_+$ for the value function/policy respectively.

To calculate the update target $L^{\pi,\mu,n}(x_0,a_0) = \sum_{t=0}^{n-1} \gamma^t r_t + V_{\theta^\prime}(x_n)$ with partial trajectory $(x_t,a_t,r_t)_{t=0}^n$ along with the target value network $V_{\theta^\prime}(x)$. We apply the target network technique to stabilizie the update, where $\theta^\prime$ is a delayed version of the major network $\theta$ and is updated as $\theta^\prime = \tau \theta^\prime + (1-\tau) \theta$ where $\tau=0.995$.

\begin{table}[t]
\caption{Summary of the performance of algorithmic variants across benchmark tasks. We use \emph{uncorrected} to denote prioritized sampling without IS corrections. Return-based \gls{SIL} is represented as \gls{SIL} with $n=\infty$. For each task, algorithmic variants with top performance are highlighted (two are highlighted if they are not statistically significantly different). Each entry shows $\text{mean} \pm \text{std}$ performance.}
\begin{small}
\begin{sc}
\begin{tabular}{C{1.in} *1{C{.6in}} *2{C{.6in}} *2{C{.6in}} *1{C{.6in}} *1{C{.6in}}}\toprule[1.5pt]
\bf Tasks & \bf \gls{SIL} $n=5$ & \bf \gls{SIL} $n=5$ (uncorrected)  & \bf \gls{SIL} $n=1$ (uncorrected) & \bf $5$-step & \bf $1$-step & \bf \gls{SIL} $n=\infty$ \ \\\midrule
DMWalkerRun & $\mathbf{642 \pm 107}$ & $\mathbf{675 \pm 15}$ & $500 \pm 138$ & $246 \pm 49$ & $274\pm 100$ & $320 \pm 111$ \\ 
DMWalkerStand & $\mathbf{979 \pm 2}$ & $947 \pm 18$ & $899 \pm 55$ &$749 \pm 150$ & $487 \pm 177$ & $748 \pm 143$ \\ 
DMWalkerWalk & $731 \pm 151$ & $622 \pm 197$ & $601 \pm 108$ &$\mathbf{925 \pm 10}$ & $793 \pm 121$ & $398 \pm 203$ \\ 
DMCheetahRun & $\mathbf{830 \pm 36}$ & $597 \pm 64$ & $702 \pm 72$ & $553 \pm 92$ & $643\pm 83$ & $655 \pm 59$ \\
Ant & $\mathbf{4123 \pm 364}$ & $3059 \pm 360$ & $\mathbf{3166 \pm 390}$ & $1058 \pm 281$ & $3968 \pm 401$ & $3787 \pm 411$\\
HalfCheetah & $8246 \pm 784$ & $9976 \pm 252$ & $\mathbf{10417 \pm 364}$ & $6178 \pm 151$ & $\mathbf{10100 \pm 481}$ & $8389\pm 386$\\
Ant(B) & $\mathbf{2954 \pm 54}$ & $1690 \pm 564$ & $1851 \pm 416$ & $\mathbf{2920 \pm 84}$ & $1866 \pm 623$ & $1884 \pm 631$\\
HalfCheetah(B) & $\mathbf{2619 \pm 129}$ & $\mathbf{2521 \pm 128}$ & $2420 \pm 109$ & $1454 \pm 338$ & $2544 \pm 31$ & $2014 \pm 378$ \\
\bottomrule
\end{tabular}
\end{sc}
\end{small}
\vskip -0.1in
\label{table:summary}
\end{table}

\begin{table}[t]
\caption{Comparison between different replay schemes. For each task, algorithmic variants with top performance are highlighted (two are highlighted if they are not statistically significantly different). Each entry shows $\text{mean} \pm \text{std}$ performance.}
\centering
\begin{small}
\begin{sc}
\begin{tabular}{C{1.in} *1{C{1.in}} *2{C{1.in}}}\toprule[1.5pt]
\bf Tasks & \bf \gls{SIL} $n=5$ & \bf \gls{SIL} $n=5$ (uncorrected)  & \bf \gls{SIL} $n=5$ (no priority)  \ \\\midrule
DMWalkerRun & $\mathbf{642 \pm 107}$ & $\mathbf{675 \pm 15}$ & $424 \pm 127$ \\ 
DMWalkerStand & $\mathbf{979 \pm 2}$ & $947 \pm 18$ & $634 \pm 184$ \\ 
DMWalkerWalk & $\mathbf{731 \pm 151}$ & $622 \pm 197$ & $\mathbf{766 \pm 103}$ \\ 
DMCheetahRun & $\mathbf{830 \pm 36}$ & $597 \pm 64$ & $\mathbf{505 \pm 182}$ \\
Ant & $\mathbf{4123 \pm 364}$ & $3059 \pm 360$ & $4358 \pm 496$\\
HalfCheetah & $8246 \pm 784$ & $\mathbf{9976 \pm 252}$ & $8927 \pm 596$\\
Ant(B) & $\mathbf{2954 \pm 54}$ & $1690 \pm 564$ & $\mathbf{2910 \pm 88}$\\
HalfCheetah(B) & $\mathbf{2619 \pm 129}$ & $\mathbf{2521 \pm 128}$ & $2284 \pm 85$ \\
\bottomrule
\end{tabular}
\end{sc}
\end{small}
\vskip -0.1in
\label{table:priority}
\end{table}

\subsection{Additional experiment results}

\begin{figure}[h]
\centering
\subfigure[\textbf{DMWalkerRun}]{\includegraphics[width=.24\linewidth]{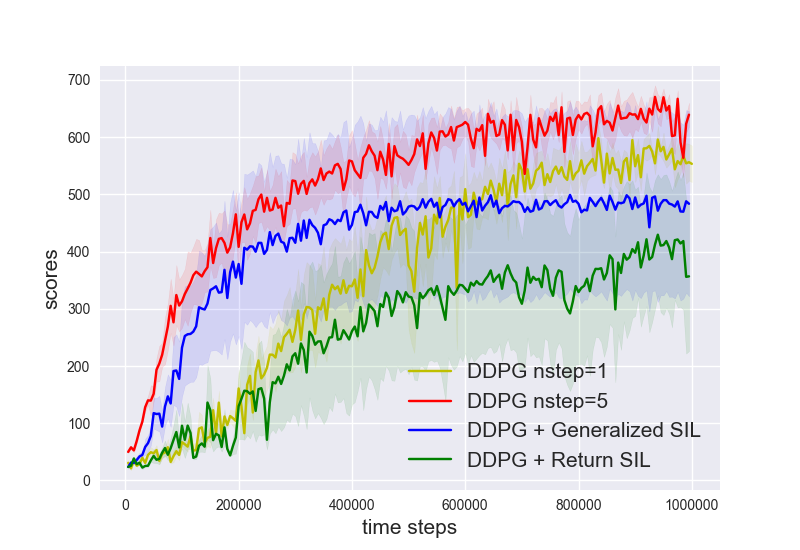}}
\subfigure[\textbf{DMWalkerStand}]{\includegraphics[width=.24\linewidth]{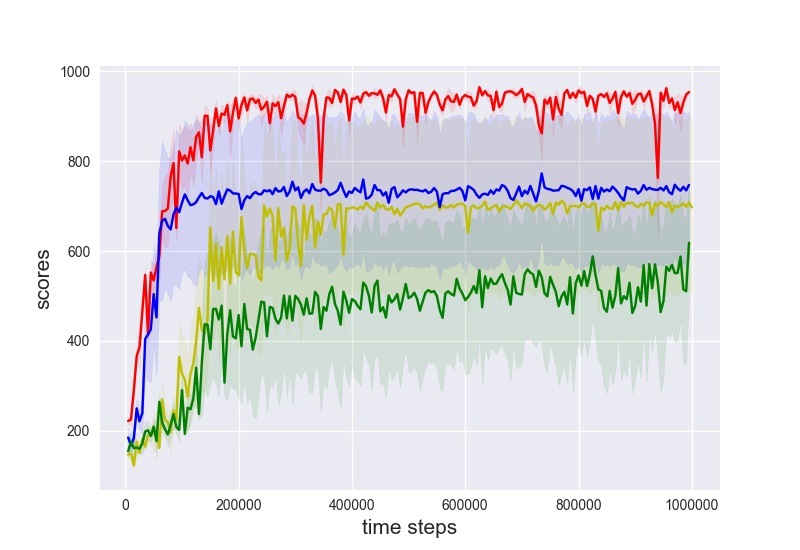}}
\subfigure[\textbf{Ant}]{\includegraphics[width=.24\linewidth]{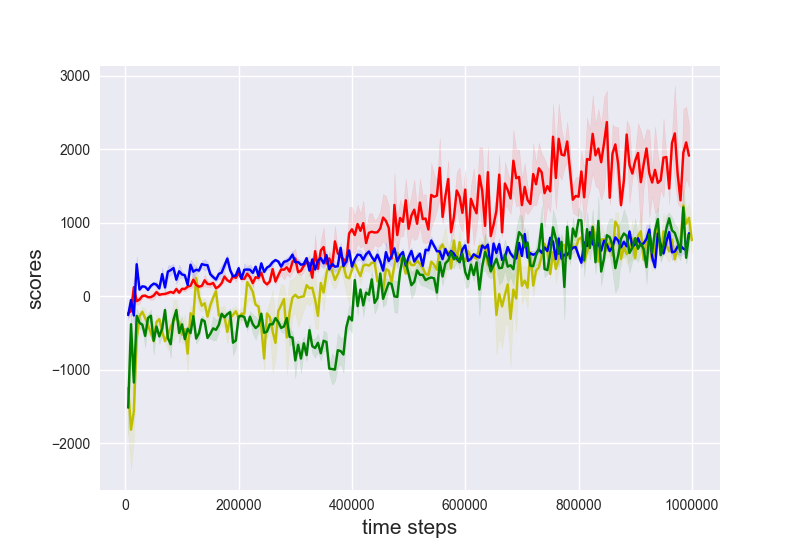}}
\subfigure[\textbf{HalfCheetah}]{\includegraphics[width=.24\linewidth]{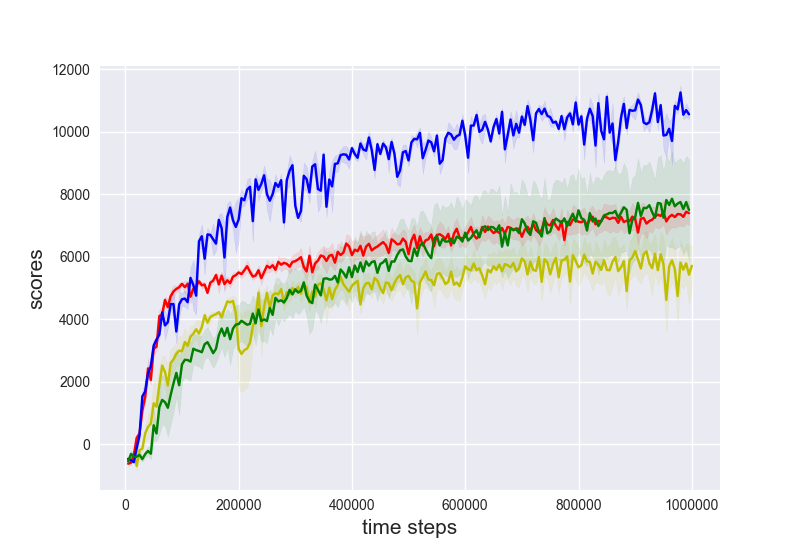}}
\caption{\small{Standard evaluations on $4$ simulation tasks for \gls{DDPG} baselines. Different colors represent different algorithmic variants. Each curve shows the $\text{mean}\pm 0.5 \text{std}$ of evaluation performance during training, averaged across $3$ random seeds. The x-axis shows the time steps and the y-axis shows the cumulative returns.}}
\label{figure:ddpg}
\end{figure}

\paragraph{Comparison across related baselines.} 
We make clear the comparison between related baselines in Table \ref{table:summary}. We present results for $n$-step \gls{TD3} with $n\in\{1,5\}$; \gls{TD3} with generalized \gls{SIL} with $n=5$ and its variants with different setups for prioritized sampling; \gls{TD3} with return-based \gls{SIL} ($n=\infty$). We show the results across all $8$ tasks - in each entry of Table \ref{table:summary} we show the $\text{mean} \pm \text{std}$ of performance averaged over $3$ seeds. The performance of each algorithmic variant is the average testing performance of the last $10^4$ training steps (from a total of $10^6$ training steps). The best algorithmic variant is highlighted in bold. We see that in general generalized \gls{SIL} with $n=5$ performs the best.

\paragraph{Ablation on the prioritized sampling.}
In prioritized sampling \citep{schaul2015prioritized}, when the tuples $d = (x_i,a_i,r_i)_{i=0}^n \in \mathcal{D}$ are sampled with priorities $s_d$, it is sampled with probability $p(d) \propto s_d^\alpha$. During updates, the IS correction consists in optimizing the loss $\mathbb{E}_{d}[w_d l_d]$ where $l_d$ is the loss computed from tuple $d$ and the IS correction weight $w_d = (N \cdot p_d)^{-\beta}$ where $N$ is the number of tuples in the buffer $\mathcal{D}$.

We compare several prioritized sampling variants of generalized \gls{SIL} in Table \ref{table:priority}. There are three variants: \gls{SIL} $n=5$ with both prioritized sampling ($\alpha=0.6$) and IS correction ($\beta=0.1$); \gls{SIL} $n=5$ with prioritized sampling ($\alpha=0.6$) only and without IS correction ($\beta=0.0$); \gls{SIL} $n=5$ with no prioritized sampling ($\alpha=\beta=0.0$). The performance setup in Table \ref{table:priority}
is the same as in Table \ref{table:summary}. It can be seen from Table \ref{table:priority} that generalized \gls{SIL} performs the best with full prioritized sampling.

\paragraph{Results on \gls{DDPG}.}
\gls{DDPG} is a baseline actor-critic algorithm with a deterministic actor \citep{lillicrap2015continuous}. Compared to \gls{TD3}, \gls{DDPG} does not adopt a double-critic approach \citep{fujimoto2018addressing} and suffers from over-estimation bias of the Q-function \citep{van2016deep}.

We present the baseline evaluation result of \gls{DDPG} in Figure \ref{figure:ddpg}, where we show the results for a few variants: \gls{DDPG} with $n$-step update, $n\in\{1,5\}$; \gls{DDPG} with generalized \gls{SIL} $n=5$ and \gls{DDPG} with return-based \gls{SIL} ($n=\infty$). We see that the performance gains of \gls{DDPG} with generalized \gls{SIL} $n=5$ are not as significant - indeed, overall \gls{DDPG} with $n=5$ has the best performance. We speculate that this is partly due to the over-estimation bias of \gls{DDPG}: the formulation of generalized \gls{SIL} is motivated by shifting the fixed point  $Q^\pi$ with an positive bias. The baseline algorithm benefits the most from generalized \gls{SIL} when indeed in practice $Q_\theta\approx Q^\pi$. However, this is not the case for \gls{DDPG} as the algorithm already has high positive bias in that $Q^\theta > Q^\pi$, which reduces the potential gains that come from generalized \gls{SIL}.

\end{document}